\documentclass{article}

     \usepackage[preprint, nonatbib]{neurips_2020}

\usepackage[utf8]{inputenc} 
\usepackage[T1]{fontenc}    
\usepackage{hyperref}       
\usepackage{url}            
\usepackage{booktabs}       
\usepackage{amsfonts}       
\usepackage{nicefrac}       
\usepackage{microtype}      

\title{Rate-adaptive model selection over a collection of black-box contextual bandit algorithms
}

\author{%
  Aurélien F. Bibaut \\
   \And
   Antoine Chambaz
   \And 
   Mark J. van der Laan
}

\usepackage{mathtools}
\mathtoolsset{showonlyrefs,showmanualtags}
\usepackage{amssymb,amsmath,amsthm}
\usepackage{color}
\usepackage{algorithm}
\usepackage{dsfont}
\usepackage{algpseudocode}
\usepackage{mathtools} 
\usepackage{fullpage}
\usepackage{natbib}
\usepackage{bm}
\usepackage{multirow}
\usepackage{array}
\usepackage{float}

\usepackage{caption}
\usepackage{subcaption}

\newtheorem{theorem}{Theorem}
\newtheorem{theorem*}{Theorem}
\newtheorem{corollary}{Corollary}
\newtheorem{assumption}{Assumption}

\newtheorem{remark}{Remark}
\newtheorem{definition}{Definition}
\newtheorem{lemma}{Lemma}

\newcommand{\Ind}{\textbf{1}}
\newcommand{\Ref}{\mathrm{ref}}

\newcommand{\xplr}{\mathrm{xplr}}
\newcommand{\xplt}{\mathrm{xplt}}

\newcommand{\argmin}{\mathop{\arg\min}}

\makeatletter
\newcommand*\rel@kern[1]{\kern#1\dimexpr\macc@kerna}
\newcommand*\widebar[1]{%
  \begingroup
  \def\mathaccent##1##2{%
    \rel@kern{0.8}%
    \overline{\rel@kern{-0.8}\macc@nucleus\rel@kern{0.2}}%
    \rel@kern{-0.2}%
  }%
  \macc@depth\@ne
  \let\math@bgroup\@empty \let\math@egroup\macc@set@skewchar
  \mathsurround\z@ \frozen@everymath{\mathgroup\macc@group\relax}%
  \macc@set@skewchar\relax
  \let\mathaccentV\macc@nested@a
  \macc@nested@a\relax111{#1}%
  \endgroup
}
\makeatother

\usepackage{xr}
\AtEndDocument{\refstepcounter{theorem}\label{finaltheorem}}
\AtEndDocument{\refstepcounter{lemma}\label{finallemma}}
\AtEndDocument{\refstepcounter{proposition}\label{finalproposition}}
\AtEndDocument{\refstepcounter{figure}\label{finalfigure}}
\begin{document}
\maketitle

\begin{abstract}
We consider the model selection task in the stochastic contextual bandit setting. Suppose we are given a collection of base contextual bandit algorithms. We provide a master algorithm that combines them and achieves the same performance, up to constants, as the best base algorithm would, if it had been run on its own. Our approach only requires that each algorithm satisfy a high probability regret bound.

Our  procedure is  very  simple and essentially does the following:  for  a  well chosen  sequence  of probabilities $(p_{t})_{t\geq 1}$,  at each round $t$, it  either chooses at random which candidate  to follow (with probability $p_{t}$) or  compares, at the
  same internal sample size for each candidate, the cumulative reward of each,
  and selects the one that wins the comparison (with probability $1-p_{t}$).

To the best of our knowledge, our proposal is the first one to be rate-adaptive for a collection of general black-box contextual bandit algorithms: it achieves the same regret rate as the best candidate.

We demonstrate the effectiveness of our method with simulation studies.
\end{abstract}

\section{Introduction}\label{section:introduction}

Contexual bandit (CB) learning is the repetition of the following steps, carried out by a an agent $\mathcal{A}$ and an environment $\mathcal{E}$.
\begin{enumerate}
\item the environment presents the agent a context $X \in \mathcal{X}$,
\item the agent chooses an action $A \in \{1,\ldots,K\}$,
\item the environment presents the learner the reward $Y$ corresponding to action $A$.
\end{enumerate}
The goal of the agent is to accumulate the highest possible cumulative reward over a certain number of rounds $T$. The relative performance of existing CB algorithms depends on the environment $\mathcal{E}$: for instance some algorithms are best suited for settings where the reward structure is linear (LinUCB), but can be outperformed by greedy algorithms when the reward structure is more complex. It would therefore be desirable to have a procedure that is able to identify, in a data-driven fashion, which one of a pool of base CB algorithms is best suited for the environment at hand. This task is referred to as model selection. In batch settings and online full information settings, model selection is a mature field, with developments spanning several decades \citep{Stone74, Lepski90, Lepski91, gyorfi2002, dudoit_vdL2005, massart2007, benkeser2018}. Cross-validation is now the standard approach used in practice, and it enjoys solid theoretical foundations \citep{devroye-lugosi2001, gyorfi2002, dudoit_vdL2005, benkeser2018}.

Literature on model selection in online learning under bandit feedback is more recent and sparser. This owes to challenges specific to the bandit setting.  Firstly, the bandit feedback structure implies that at any round, only the loss (here the negative reward) corresponding to one action can be observed, which implies that the loss can be observed only for a subset of the candidate learners (those which proposed the action eventually chosen). Any model selection procedure must therefore address the question of how to decide which base learner to follow at each round (the \textit{allocation challenge}), and how to pass feedback to the base learners (the \textit{feedback challenge}). A tempting approach to decide how to allocate rounds to different base learners is to use a standard multi-armed bandit (MAB) algorithm as a meta-learner, and treat the base learners as arms. This approach fails because, unlike in the usual MAB setting, the reward distribution of the arms changes with the number of times they get played: the more a base learner gets chosen, the more data it receives, and the better its proposed policy (and therefore expected reward) becomes. This exemplifies the \textit{comparability challenge}: how to compare the candidate learners based on the available data at any given time?

Existing approaches solve these challenges in differents ways. 
We saw essentially two types of solutions in the existing literature, represented on the one hand by the OSOM algorithm of \cite{chatterji2019osom} and the ModCB algorithm of \cite{foster2019}, and on the other hand by the CORRAL algorithm of \cite{agarwal17}, and the stochastic CORRAL algorithm, an improved version thereof introduced by \cite{pacchiano2020model}. 

In OSOM and ModCB, the base learners learn policies in policy classes that form a nested sequence, which can be ordered from least complex to most complex. Their solution to the \textit{allocation challenge} is to start by using the least complex algorithm, and move irreversibly to the next one if a goodness-of-fit test indicates its superiority. The goodness-of-fit tests uses all the data available to compute fits of the current and next policy, and compares them. This describes their solution to the \textit{feedback challenge} and the \textit{comparability challenge}. 

CORRAL variants take another route. They use an Online Mirror Descent (OMD) based master algorithm  that samples alternatively which base learner to follow, and gradually phases out the suboptimal ones. In that sense, their allocation strategy resembles the one of a MAB algorithm.  The \textit{comparability issue} arises naturally in the context of an OMD meta-learner, which can be understood easily with an example. Suppose that we have two base algorithms $\mathcal{A}(1)$ and $\mathcal{A}(2)$, and that $\mathcal{A}(1)$ has better asymptotic regret thant $\mathcal{A}(2)$. It can happen that either by chance ($\mathcal{A}(1)$ plays unlucky rounds) or by design (e.g. $\mathcal{A}(1)$ explores a lot in early rounds), $\mathcal{A}(1)$ fares worse than $\mathcal{A}(2)$ initially. As a result, the master would initially give a lesser weight to $\mathcal{A}(1)$ than to $\mathcal{A}(2)$, with the result that at some time $t$, the policy proposed by $\mathcal{A}(1)$ is based on a much smaller internal sample size than the policy proposed by $\mathcal{A}(2)$. As a result, at $t$, even though $\mathcal{A}(1)$ is asymptotically better than $\mathcal{A}(2)$, the losses of $\mathcal{A}(1)$ are worse than the losses of $\mathcal{A}(2)$, which accentuates the data-starvation of $\mathcal{A}(1)$ and can lead to $\mathcal{A}(1)$ never recovering from its early underperformance. The issue described here is that the losses used for the OMD weights update are not comparable across candidates, as they are based on policies informed by significantly different internal sample sizes. CORRAL can be viewed as the solution to the \textit{comparability challenge} in the context of an OMD master: by using gentle weight updates (as opposed to the more aggressive weight updates of Exp3 for instance) and by regularly increasing the learning rate of base learners of which the weight drops too low, CORRAL prevents the base algorithm data-starvation phenomenon. The two CORRAL variants differ in their solution to the \textit{feedback challenge}. 
The original CORRAL algorithm \citep{agarwal17} passes, at each round, importance weighted losses to the master and to all base learners. In contrast,\cite{pacchiano2020model}'s stochastic CORRAL passes unweighted losses to each base algorithm, but only at the time they get selected.

Guarantees in \cite{chatterji19a} and \cite{foster2019} rely on the so-called \textit{realizability} assumption, which states that at least one of the candidate policy classes contains $\pi_0(\mathcal{E})$, the optimal measurable policy under the current environment. \cite{chatterji19a} show that their approach achieves the minimax regret rate for the smallest policy class that contains $\pi_0(\mathcal{E})$. \cite{foster2019} consider linear policy classes and show that their algorithm achieves regret no larger than $\widetilde{\mathcal{O}}(T^{2/3} d_*^{1/3})$ and $\widetilde{\mathcal{O}}(T^{3/4} + \sqrt{T d_*})$ where $d^*$ is the dimension of the smallest policy class that contains $\pi_0(\mathcal{E})$. This is optimal if $d_* \geq \sqrt{T}$. In CORRAL variants, if one the the $J$ base algorihtms has regret $O(T^\alpha)$, the master achieves regret $\widetilde{\mathcal{O}}(J/T + T \eta +  T \eta^{(1-\alpha) / \alpha})$, with $\eta$ the initial learning rate of the master. The learning rate $\eta$ can be optimized so that this regret bound becomes $\widetilde{\mathcal{O}}(J^{1-\alpha} T^\alpha)$, that is, up to log factors, the upper bound on regret of that base algorithm. As pointed out  by \citet{agarwal17}, and as can be seen from the regret bound restated here, CORRAL presents an important caveat: the learning rate must be tuned to the rate of the base algorithm one wishes to compete with. This is not an issue when working with a collection of algorithms with same regret upper bound, and in that case CORRAL offers protection against model misspecification. However when base learners have different regret rates, CORRAL fails to adapt to the rate of the optimal algorithm.

In this article, we propose a master algorithm that allows to work with general off-the-shelf (contextual) bandit algorithhms, and achieves the same regret rate as the best of them. Our theoretical guarantees improve upon OSOM \citep{chatterji2019osom} and ModCB \citep{foster2019} in the sense that our algorithm allows to work with a general collection of bandit algorithms, as opposed to a collection of algorithms based on a nested sequence of parametric reward models. It improves upon CORRAL variants in the sense that it is rate-adaptive. 
Our master algorithm can be  described as follows: for a well chosen sequence $(p_{t})_{t\geq 1}$ of exploration probabilities, at each time $t$, the master either samples a base algorithm uniformly at random and follows its proposal (with probability $p_{t}$), or it picks the base algorithm that maximizes a  certain criterion based on past  performance (with an exploitation probability of $1-p_{t}$).
Each algorithm receives feedback only if it gets played by the master. The crucial idea is to compare the performance of base algorithms at the same internal time. At global time $t$, the $J$ algorithms are at internal times $n(1,t),\ldots, n(J,t)$ (with $n(1,t)+\ldots+ n(J,t) = t$). We compare them based on their $\underbar{n}(t) := \min_{j \in [J]} n(j,t)$ first rounds, thus ensuring a fair comparison.

We organize the article as follows. In section \ref{section:problem_setting}, we formalize the setting consisting of a master algorithm allocating rounds to base algorithms. In section \ref{section:alg_description}, we present our master algorithm, EnsBFC (Ensembling Bandits by Fair Comparison). We present its theoretical guarantees in section \ref{section:regret_guarantees}. We show in section \ref{section:hp_reg_bounds_known_algs} that many well-known existing bandit algorithms satisfy the assumption of our main theorem. We give experimental validation of our claims in section \ref{section:simulation_study}.

\section{Problem setting}\label{section:problem_setting}

\subsection{Master data and base algorithms internal data}

A master algorithm $\mathcal{M}$ has access to $J$ base contextual bandit algorithms $\mathcal{A}(1),\ldots,\mathcal{A}(J)$. At any time $t$, the master observes a context vector $X(t) \in \mathcal{X} \subset \mathbb{R}^d$, selects the index $\widehat{J}(t)$ of a base algorithm, and draws an action $A(t) \in [K] := \{1,\ldots,K\}$, following the policy of the selected base algorithm. The environment presents the reward $Y(t)$ corresponding to action $A(t)$.  We distinguish two types of rounds for the master algorithm: exploration rounds and exploitation rounds. We will cover in more detail further down the definition of each type of round. We let $D(t)$ be the indicator of the event that round $t$ is an exploration round. The data  collected at time $t$ by the master algorithm is $Z(t):=(D(t), \widehat{J}(t), X(t), A(t), Y(t))$. We denote $O(t):= (X(t), A(t), Y(t))$ the subvector of $Z(t)$ corresponding to the triple context, action, reward at time $t$. 
We denote $\mathcal{F}(t) := \sigma(Z(1),\ldots,Z(t))$, the filtration induced by the first $t$ observations. We suppose that contexts are independent and identically distributed (i.i.d.) and that the conditional distribution of rewards given actions and contexts is fixed across time points.

After each round $t$, the master passes the triple $(X(t), A(t), Y(t))$ to  base algorithm $\widehat{J}(t)$, which increments the internal time $n(\widehat{J}(t), t)$ of algorithm $\widehat{J}(t)$ by 1, and leaves unchanged the internal time of the other algorithms. For any $j \in [J]$, $n\geq 1$, we denote $\widetilde{O}(j,n) = (\widetilde{X}(j,n), \widetilde{A}(j,n), \widetilde{Y}(j,n))$ the triple collected by base algorithm $j$ at its internal time $n$. Making this more formal, we define the internal time of $j$ at global time $t$ as $n(j,t):=\sum_{\tau=1}^t \Ind(\widehat{J}(\tau)=j)$, that is the number of times $j$ has been selected by the master up till global time $t$. We define the reciprocal of $n(j,t)$ as $t(j,n):= \min \{ t \geq 1: n(j,t) = n\}$, that is the global time at which the internal time of $j$ was updated from $n-1$ to $n$. We can then formally define $\widetilde{O}(j,n)$ as $\widetilde{O}(j,n) :=  (\widetilde{X}(j,n), \widetilde{A}(j,n), \widetilde{Y}(j,n)) := (X(t(j,n)), A(t(j,n)), Y(t(j,n))$. We denote $\widetilde{\mathcal{F}}(j,n):= \sigma(\widetilde{O}(j,1),\ldots,\widetilde{O}(j,n))$ the filtration induced by the first $n$ observations of algorithm $\mathcal{A}(j)$. 

Let $n^{\xplr}(j,t) := \sum_{\tau=1}^t \Ind(\widehat{J}(\tau) = j, D(\tau) = 1)$ and $n^{\xplt}(j,t) := \sum_{\tau=1}^t \Ind(\widehat{J}(\tau) = j, D(\tau) = 0)$, the number of exploration and exploitation rounds $j$ was selected up till global time $t$. Note that $n(j,t) = n^{\xplr}(j,t) + n^{\xplt}(j,t)$. Define $\underbar{n}(t) := \min_{j \in [J]} n(j,t)$, $\underbar{n}^{\xplr}(t) := \min_{j \in [J]} n^{\xplr}(j,t)$, and $\underbar{n}^{\xplt}(t) := \min_{j \in [J]} n^{\xplt}(j,t)$.

\subsection{Policies and base algorithm regret}

A policy $\pi:[K] \times \mathcal{X} \rightarrow [0,1]$ is a conditional distribution over actions given a context, or otherwise stated, a mapping from contexts to a distribution over actions. So as to define the value and the risk of a policy, we introduce an triple of reference $(X^\Ref, A^\Ref, Y^\Ref)$ such that $X^\Ref$ has same distribution as any context $X(t)$, $Y^\Ref|A^\Ref,X^\Ref$ has same law as $Y(t)|A(t),X(t)$ for any $t$, and $A^\Ref|X^\Ref \sim \pi^\Ref(\cdot,X^\Ref)$, where $\pi^\Ref(a,x):=1/K$ for every $a$ and $x$. We introduce what we call the value loss $\ell$, defined for any policy $\pi$ and triple $o \in \mathcal{X} \times [K] \times \mathbb{R}$ as $\ell(\pi)(o):= -y \pi(a,w) / \pi^\Ref(a,w)$. We then define the risk of $\pi$ as $R(\pi):=E[\ell(\pi)(O^\Ref)]$. We will use that $-R(\pi) = E[Y^\Ref \pi(A^\Ref,X^\Ref)/\pi^\Ref(A^\Ref,X^\Ref)] = E[\sum_{a=1}^K \pi(a|X^\Ref) E[Y^\Ref|A^\Ref=a,X^\Ref]]$, where the latter quantity is the value of $\pi$, that is the expected reward per round one would get if one carried out $\pi$ under environment $\mathcal{E}$. We denote it $\mathcal{V}(\pi, \mathcal{E})$.

We denote $\pi(j,n)$ the policy proposed by $\mathcal{A}(j)$ at its internal time $n$. For any $x \in \mathcal{X}$, $\pi(j,n)(\cdot,x)$ is an $\widetilde{\mathcal{F}}(j,n-1)$-measurable distribution over $[K]$.
We suppose that each algorithm $\mathcal{A}(j)$ operates over a policy class $\Pi_j$. The regret of $\mathcal{A}(j)$ over its first $n$ rounds is defined as $\mathrm{Reg}(j,n):= \sum_{\tau=1}^n (\mathcal{V}^*_j(\mathcal{E}) - \widetilde{Y}(j,\tau))$, with $\mathcal{V}^*_j(\mathcal{E}) := \sup_{\pi \in \Pi_j} \mathcal{V}(\pi, \mathcal{E})$. We define the cumulative conditional regret as $\mathrm{CondReg}(j,n) := \sum_{\tau=1}^n (\mathcal{V}^*_j(\mathcal{E}) - E[\widetilde{Y}(j,\tau)|\widetilde{\mathcal{F}}_{\tau-1}]) = n( \widebar{R}(j,n) - R^*_j)$, with $R^*_j = - \mathcal{V}^*_j(\mathcal{E})$ and $\widebar{R}(j,n) = n^{-1}  \sum_{\tau=1}^n R(\pi(j, \tau))$, where the identity follows from the fact that $E[\widetilde{Y}(j,\tau)|\widetilde{\mathcal{F}}(j,\tau-1)] = \mathcal{V}(\pi(j, \tau), \mathcal{E}) =  - R(\pi(j,\tau))$. We define the pseudo regret as $\mathrm{pseudoReg}(j,n) := E[\mathrm{Reg}(j,n)]$.

\subsection{Master regret and rate adaptivity}\label{subsection:rate_adaptivity}

We let $\mathcal{V^*}(\mathcal{E}) := \max_{j \in [J]} \mathcal{V}^*_j(\mathcal{E})$, the optimal value across all policy classes $\Pi_1,\ldots,\Pi_J$, and similarly, we denote $R^* := \min_{j \in [J]} R^*_j$, the optimal risk across $\Pi_1,\ldots,\Pi_J$. We define the regret of the master as $\mathrm{Reg}(t):= \sum_{\tau=1}^t \mathcal{V}^*(\mathcal{E}) -  Y(t)$, and the conditional regret as $\mathrm{CondReg}(t) := \sum_{\tau=1}^t \mathcal{V}^*(\mathcal{E}) -  E[Y(\tau)|\mathcal{F}(\tau-1)]$.

The bandit literature gives upper bounds on either $\mathrm{Reg}(j,n)$ or $\mathrm{CondReg}(j,n)$ where the dependence in $n$ is of the form $\widetilde{\mathcal{O}}(n^{1-\beta_j})$, for some $\beta_j \in (0,1)$. (We denote $a_n = \widetilde{\mathcal{O}}(b_n)$ if $a_n = \mathcal{O}(b_n (\log n)^{\gamma})$ for some $\gamma > 0$.) While $\beta_j$ is known, it is not the case for $\mathcal{V}^*_j(\mathcal{E})$, the asymptotic value of (the policy proposed by) $\mathcal{A}(j)$.

As a necessary requirement, a successful meta-learner should achieve asymptotic value $\mathcal{V}^*(\mathcal{E})$. A second natural requirement is that it should have as good regret guarantees as the best algorithm in the subset $\mathcal{J}:=\{j \in [J] : \mathcal{V}^*_j(\mathcal{E}) = \mathcal{V}^*(\mathcal{E})\}$ of algorithms with optimal asymptotic value. We say that a master algorithm is \textit{rate-adaptive} if it achieves these two requirements.

\begin{definition}[Rate adaptivity]
Suppose that base algorithms have known regret (or conditional regret, or pseudo regret) upper bounds $\widetilde{\mathcal{O}}(n^{1-\beta_1}),\ldots,\widetilde{\mathcal{O}}(n^{1-\beta_J})$. Let $\beta(1) = \max_{j \in \mathcal{J}} \beta_j$, the rate exponent corresponding to the fastest upper bound rate among algorithms with optimal limit value $\mathcal{V}^*(\mathcal{E})$.

We say that the master is rate-adaptive in regret (or conditional regret, or pseudo regret), up to logarithmic factors, if it holds that $\mathrm{Reg}(t) = \widetilde{\mathcal{O}}(t^{1-\beta(1)})$ (or $\mathrm{CondReg}(t) = \widetilde{\mathcal{O}}(t^{1-\beta(1)})$, or $\mathrm{pseudoReg}(t) = \widetilde{\mathcal{O}}(t^{1-\beta(1)})$).
\end{definition}

\begin{remark}
A natural setting where several base algorithms converge to the same value $\mathcal{V}^*(\mathcal{E})$ is when several of the candidate policy classes contain the optimal measurable policy $\pi_0(\mathcal{E})$, that is when the realizability assumption is satisfied for several base policy classes.
\end{remark}

\begin{remark}
Suppose that rates $\widetilde{\mathcal{O}}(n^{1-\beta_1}),\ldots,\widetilde{\mathcal{O}}(n^{1-\beta_J})$ are minimax  optimal (up to logarithmic factors) for the policy classes $\Pi_1,\ldots,\Pi_J$, and that at least one  class contains $\pi_0(\mathcal{E})$. Then, in this context, rate adaptivity means that the master achieve the best minimax rate among classes that contain $\pi_0(\mathcal{E})$. In this context, \emph{rate-adaptivity} coincides with the notion of \emph{minimax adaptivity} from statistics' model selection literature (see e.g. \cite{massart2007, gine_nickl_2015}).
\end{remark}

\begin{remark}
OSOM \citep{chatterji2019osom} and ModCB \cite{foster2019} are minimax adaptive (and thus rate-adaptive) under the condition that $\pi_0$ belongs to at least one of the policy classes (that is under the realizability assumption). CORRAL and stochastic CORRAL are not rate-adaptive.
\end{remark}

\section{Algorithm description}\label{section:alg_description}

Our master algorithm $\mathcal{M}$ can be described as follows. At each global time $t \geq 1$, $\mathcal{M}$ selects a base algorithm index $\widehat{J}(t)$ based on past data, observes the context $X(t)$, draws an action $A(t)$ conditional on $X(t)$ following the policy $\pi(\widehat{J}(t), n(\widehat{J}(t), t-1))$ proposed by $\mathcal{A}(\widehat{J}(t))$ at its current internal time, carries out action $A(t)$ and collects reward $Y(t)$. At the end of round $t$, $\mathcal{M}$ passes the triple $(X(t), A(t), Y(t))$ to $\mathcal{A}(\widehat{J}(t)))$, which then increments its internal time and updates its policy proposal based on the new datapoint.

To fully characterize $\mathcal{M}$ it remains to describe the mechanism that produces $\widehat{J}(t)$. We distinguish exploration rounds and exploitation rounds. We determine if round $t$ is to be an exploration round by drawing, independently from the past $\mathcal{F}(t-1)$, the exploration round indicator $D(t)$ from a Bernoulli law with probability $p_t$, which we will define further down. During an exploration round (if $D(t)=1$), we draw $\hat{J}(t)$ independently of $\mathcal{F}(t-1)$, from a uniform distribution over $[J]$. During an exploitation round (if $D(t) =0$), we draw $\widehat{J}(t)$ based on a criterion depending on the past rewards of base algorithms. Let us define this criterion.

Let $\widehat{R}(j,n) := - n^{-1} \sum_{\tau=1}^n Y(j,\tau)$, the mean of negative rewards collected by algorithm $j$ up till its internal time $n$. For any $n \geq 1$, define the algorithm selector $\widehat{j}(n, \widehat{R}(1,n),\ldots\widehat{R}(J,n), c_1) := \argmin \{\widehat{R}(j,n) + c_1 n^{-\beta_j} : j \in [J]\}$, with $c_1 > 0$ a tuning parameter. When there is no ambiguity, we will use the shorthand notation $\widehat{j}(n)$. The selector $\widehat{j}(n)$ compares every base algorithm at the same internal time, and picks the one that minimizes the sum of the estimated risk at internal time $n$ plus the theoretical regret upper bound rate $n^{-\beta_j}$. If $D(t)=0$, we let $\hat{J}(t):=\widehat{j}(\underline{n}^{\xplr}(t))$, that is we compare the base algorithms at a common internal time equal to the highest common number of exploration rounds each base has been called until $t$.

If any base algorithm $j$ has average risk converging to some $R^*_j > R^*$, the regret of an exploration step is $O(1)$ in expectation. If we want the regret of the master with respect to (w.r.t.) $R^*$ to be $\mathcal{O}(t^{-\beta(1)})$, we need the exploration probability $p_t$ to be $\mathcal{O}(t^{-\beta(1)})$. Because $\beta(1)$ is unknown (it depends
on $\mathcal{J}$ hence on $\mathcal{E}$ too), we make a conservative choice and we set $p_t := c_2 t^{-\overline{\beta}}$, with $\overline{\beta} := \max_{j \in [J]} \beta_j$ (a quantity available to us), where $c_2 > 0$ is a tuning parameter.

We give the pseudo code of the master algorithm $\mathcal{M}$ as algorithm \ref{alg:master} below.
\begin{algorithm}
   \caption{Ensembling Bandits by Fair Comparison (EnsBFC)}\label{alg:master}
   \renewcommand{\algorithmicrequire}{\textbf{Input:}}
\renewcommand{\algorithmicensure}{\textbf{Output:}}
\begin{algorithmic}
   \Require base algorithms $\mathcal{A}(1),\ldots,\mathcal{A}(J)$, theoretical regret per round exponents $\beta_1,\ldots,\beta_j$, tuning parameters $c_1, c_2$.
	\State Initialize risk estimators: $\widehat{R}(j,0) \gets 0$ for every $j \in [J]$.
   \For{$t \geq 1$}
   \State Draw exploration round indicator $D(t) \sim \mathrm{Bernoulli}(p_t)$.
   \If{$D(t)=1$} 
   \State Draw $\widehat{J}(t) \sim \mathrm{Unif}([J])$.
   \Else
   \State Set $\widehat{J}(t) \gets \widehat{j}(\underline{n}^{\xplr}(t), \widehat{R}(1,\underline{n}^{\xplr}(t)),\ldots\widehat{R}(J,\underline{n}^{\xplr}(t)), c_1)$.
   \EndIf
   \State Observe context $X(t)$.
   \State Sample action $A(t)$ following the policy proposed by $\mathcal{A}(\widehat{J}(t))$ at its current internal time:
   \begin{equation}
   A(t)|X(t) \sim \pi(\widehat{J}(t), n(\widehat{J}(t), t-1))(\cdot, X(t)).
   \end{equation}
   \State Collect reward $Y(t)$.
   \State Pass the triple $(X(t), A(t), Y(t))$ to $\mathcal{A}(\widehat{J}(t))$, which then updates its policy proposal and increments its internal time by 1.
   \EndFor.
\end{algorithmic}
\end{algorithm}

\section{Regret guarantees of the master algorithm}\label{section:regret_guarantees}

Our main result shows that the expected regret of the master satisfies the same theoretical upper bound with respect to $R^*$ as the best base algorithm. The main assumption is that  each base algorithm satisfies its conditional regret bound $\mathcal{O}(n^{1-\beta_j})$ with high probability. We state this requirement formally as an exponential deviation bound.
\begin{assumption}[Concentration]\label{assumption:concentration}
There exists $C_0 \geq 0$, $C_1, C_2 > 0$, $\beta_1,\ldots\beta_J \leq 1/2$, $\nu_1,\ldots, \nu_J  > 0$ such that, for any $n\geq 1$,  $j\in [J]$ and $x \in [0,1]$,
\begin{align}
P \left[ \widebar{R}(j,n) - R^*_j \geq C_0 n^{-\beta_j} + x \right] \leq C_1 \exp\left(-C_2 \times (n x^{1/\beta_j})^{\nu_j}\right), \label{eq:exp_deviation_bound}
\end{align}
and $\widebar{R}(j,n) - R^*_j \geq 0$.
\end{assumption}
We also require that the rewards be conditionally sub-Gaussian given the past. Without loss of generality, we require that they be conditionally 1-sub-Gaussian.
\begin{assumption}\label{assumption:conditional_sub_gaussian_rewards}
For all $\lambda \in \mathbb{R}$, and every $t \geq 1$, $E[\exp(\lambda(Y_t - E[Y_t|\mathcal{F}_{t-1}])|\mathcal{F}_{t-1}] \leq \exp(\lambda^2 / 2)$.
\end{assumption}
We show in the next section that the high probability regret bounds available in the literature for many well-known CB algorithms can be reformulated as an exponential deviation bound of the form \eqref{eq:exp_deviation_bound}. We can now state our main result.
\begin{theorem}[Expected regret for the master]\label{thm:regret_master}
Suppose that assumptions \ref{assumption:concentration} and \ref{assumption:conditional_sub_gaussian_rewards} hold, and recall the definition of $\beta(1)$ from subsection \ref{subsection:rate_adaptivity}.  Then, \emph{EnsBFC} is rate-adaptive in pseudo-regret, that is,
$$E\left[\sum_{t=1}^T \left(\mathcal{V}^*(\mathcal{E}) - Y(t)\right)\right] \leq C T^{1-\beta(1)},$$ for some $C > 0$ depending only on the constants of the problem. If, in addition, the regret upper bounds satisfied by the base algorithms are minimax for their respective policy classes, then EnsBFC is minimax adaptive in pseudo regret.
\end{theorem}

\begin{remark}
Assumption \ref{assumption:concentration} is met for many well-known algorithms, as we show in the following section.
\end{remark}

\begin{remark}
The $c_1 n^{-\beta_j}$ term in the criterion $\widehat{R}(j,n) + c_1 n^{-\beta_j}$ that $\widehat{j}(n)$ minimizes across $[J]$ ensures that $\widehat{R}(j,n) + c_1 n^{-\beta_j} - R^*$ is, in expectation, lower bounded by $c_1 n^{-\beta_j}$. It may be the case that, among the base algorithms that have optimal limit value $\mathcal{V}^*(\mathcal{E})$ (that is those in $\mathcal{J}$), the one that performs best in a given environment is not the one that has best regret rate upper bound $\widetilde{\mathcal{O}}(n^{1-\beta(1)})$. Enforcing this lower bound on the criterion ensures that the master picks an algorithm with optimal regret upper bound $\widetilde{\mathcal{O}}(n^{1-\beta(1)})$. We further discuss the need for such a lower bound in appendix \ref{section:lower_bound}.
\end{remark}

\begin{remark}
The rate of pseudo-regret of EnsBFC is not impacted by the specific values of the tuning parameters $c_1$ and $c_2$ (as long as they are set to constants independent of $T$), but the finite performance is. We found in our simulations that setting $c_1 = 0.5$ and $c_2 = 10$ works fine. We leave to future work the task of designing a data-driven rule of thumb to select $c_1$ and $c_2$.
\end{remark}

In the next subsection, we take a step back to put our results in perspective with the broader model selection literature.

\subsection{Comments on the nature of the result: minimax adaptivity vs. oracle equivalence}

Results in the model selection literature are essentially of two types: minimax adaptivity guarantees and oracle inequalities.

Given a collection of statistical models, a model selection procedure is said to be minimax adaptive if it achieves the minimax risk of any model that contains the ``truth''. In our setting, the statistical models are policy classes and the ``truth'' is the optimal measurable policy $\pi_0(\mathcal{E})$. A notable example of minimax adaptive model selection procedure is Lepski's method \citep{Lepski90, Lepski91}.

Consider a collection of estimators $\widehat{\theta}_1,\ldots,\widehat{\theta}_J$, and a data-generating distribution $P$, and denote $\mathcal{R}(\widehat{\theta}, P)$ the risk of any estimator $\widehat{\theta}$ under $P$. In our context, one should think of the estimators as the policies computed by the base algorithms, and of specifying $P$ as specifying $\mathcal{E}$. We say that an estimator $\widehat{\theta}$ satisfies an oracle inequality w.r.t. $\widehat{\theta}_1,\ldots,\widehat{\theta}_J$ if $\mathcal{R}(\widehat{\theta}, P) \leq (1+\epsilon) \min_{j \in [J]} \mathcal{R}(\widehat{\theta}_j, P) + \mathrm{Err}$, with $\epsilon > 0$ and $\mathrm{Err}$ an error term. Moreover, we say that the estimator $\widehat{\theta}$ is oracle equivalent if $\mathcal{R}(\widehat{\theta}, P) / \min_{j \in [J]} \mathcal{R}(\widehat{\theta}_j, P) \rightarrow 1$. Being oracle equivalent means performing as well as the best instance-dependent (that is $P$-dependent) estimator. Multi-fold cross validation yields an oracle-equivalent estimator \citep{devroye-lugosi2001, gyorfi2002, dudoit_vdL2005}.

Our guarantees are closer to the notion of minimax adaptivity than to that of oracle equivalence, and, as we pointed out earlier, coincide with it if the base algorithms are minimax w.r.t. their policy classes. Minimax adaptivity is the property satisfied by the OSOM \citep{chatterji2019osom} and ModCB \cite{foster2019}. Minimax adaptivity is a worst-case (over each base model) statement, which represents a step in the right direction. We nevertheless argue that what practioners are looking for in a model selection procedure is to get the same performance as the base learner that performs best under the environment at hand, that is oracle equivalence, like the guarantee offered by multi-fold cross-validation.

\section{High probability regret bound for some existing CB algorithms}\label{section:hp_reg_bounds_known_algs}

In this section, we recast regret guarantees for well-known CB algorithms under the form the exponential bound \eqref{eq:exp_deviation_bound} from our concentration assumption (assumption \ref{assumption:concentration}).

Recall the definitions of $\mathrm{Reg}$, $\mathrm{CondReg}$ and $\mathrm{pseudoReg}$ from section \ref{section:problem_setting}.
Observe that our concentration assumption is a high probability bound on $\widebar{R}(n) -R^* = \mathrm{CondReg}(n) / n$, the average of the conditional instantaneous regret. Although some articles provide high probability bounds directly on $\mathrm{CondReg}(n)$ (e.g. \cite{abbasi-yadkori2011}), most works give high probability bounds on $\mathrm{Reg}(n)$. Fortunately, under the assumption that rewards are conditionally sub-Gaussian (assumption \ref{assumption:conditional_sub_gaussian_rewards}), we can recover a high probability regret bound on $\mathrm{CondReg}(n)$ from a high probability regret bound on $\mathrm{Reg}(n)$ using the Azuma-Hoeffding inequality.

(In the following paragraphs, we suppose, to keep notation consistent, that $j$ is a base learner of the type considered in the paragraph).
\paragraph{UCB.} \cite[Lemma 4.9 in][]{pacchiano2020model}, itself a corollary of \cite[theorem 7 in][]{abbasi-yadkori2011} states that if the rewards are conditionally $1$-sub-Gaussian, the regret of UCB over $n$ rounds is $\mathcal{O}(\sqrt{n \log(n / \delta)})$.

\begin{corollary}[Exponential deviation bound for UCB]\label{corollary:exp_dev_UCB}
Suppose that assumption \ref{assumption:conditional_sub_gaussian_rewards} holds. Then, there exist $C_0, C_1, C_2 > 0$ such that, for all $x \geq 0$,
$P \left[ \widebar{R}(j,n) - R^*_j \geq C_0 n^{-1/2} (\log n)^{1/2} + x \right] \leq C_1 \exp(- C_2 n x^2).$
\end{corollary}

\paragraph{$\varepsilon$-greedy.} \cite{bibaut2020} consider the $\varepsilon$-greedy algorithm over a nonparametric policy class. The following result is a direct consequence of an intermediate claim in the proof \cite[thereom 4 in][]{bibaut2020}.

\begin{lemma}[Exponential deviation bound for $\varepsilon$-greedy]\label{lemma:exp_dev_eps_greedy} Consider the $\varepsilon$-greedy algorithm over a nonparametric policy class $\Pi$. Suppose that the metric entropy in $\| \cdot\|_\infty$ norm of $\Pi$ satisfies $\log N(\rho, \Pi, \|\cdot\|_\infty) =\mathcal{O}(\rho^{-p})$ for some $p > 0$, and that the exploration rate at $t$ is $\epsilon_t \propto t^{-(\frac{1}{3} \vee \frac{p}{p+1})}$. Then, there exist $C_0, C_1, C_2 > 0$ such that, for all $x \geq 0$, 
$P \left[ \widebar{R}(j,n) - R^*_j \geq C_0 t^{-\beta} + x \right] \leq C_1 \exp(-C_2 \times (n x^{1/\beta})^{2 \beta} ),$
with $\beta = \frac{1}{3} \vee \frac{p}{p+1}$.
\end{lemma}

\paragraph{LinUCB.} \cite[Theorem 3 in][]{abbasi-yadkori2011} states that LinUCB satisfies $\mathrm{CondReg}(n) = \mathcal{O}(\sqrt{n} \log(1/ \delta))$ with probability at least $1 -\delta$. We recast their bound as follows.
\begin{corollary}\label{corollary:exp_dev_linucb}
Under the conditions of \cite[theorem 3 in][]{abbasi-yadkori2011}, there exists $C_2 > 0$ such that, for all $x > 0$,
$P\left[ \widebar{R}(j,n) - R^*_j \geq x \right] \leq \exp( - C_2 (n x^2)^{1/2})$
\end{corollary}

\paragraph{ILOVETOCONBANDITS.} \cite[Theorem 2 in][]{agarwalb14} et al. states that $\mathrm{Reg}(n) = \mathcal{O}(\sqrt{n} \log(n / \delta) + \log(n / \delta))$ with probability at least $1 - \delta$. (The proof of their lemma actually states as an intermediate claim a $(1-\delta)$-probability bound on  $\mathrm{CondReg}(n)$ which can easily be shown to be $\mathcal{O}(\sqrt{n} \log(n / \delta) + \log(n / \delta))$ as well). We recast their bound as follows.
\begin{corollary}[Exponential deviation bound for ILOVETOCONBANDITS]\label{corollary:exp_dev_agarwal}
Suppose that assumption \ref{assumption:conditional_sub_gaussian_rewards} holds. Then, there exist $C_0 > 1$, $C_2 > 0$ such that, for any $x \geq 0$,
$P \left[ \widebar{R}(j,n) - R^*_j \geq C_0 n^{-1/2} \log n + x \right] \leq \exp(-C_2 n x^2)$.
\end{corollary}

\section{Simulation study}\label{section:simulation_study}

We implemented EnsBFC using LinUCB and an $\varepsilon$-greedy algorihtm as base learners, and we evaluated it under two toy environments. We considered the setting $K=2$. We chose environments $\mathcal{E}_1$ and $\mathcal{E}_2$, and the specifications of the two base algorithms such that:
\begin{itemize}
\item the $\varepsilon$-greedy has regret $\mathcal{O}(T^{2/3})$ w.r.t. the value $\mathcal{V}_0(\mathcal{E}_1)$ of the optimal measurable policy under $\mathcal{E}_1$, while LinUCB has linear regret lower bound $\Omega(T)$  w.r.t. $\mathcal{V}_0(\mathcal{E}_1)$,
\item LinUCB has regret $\mathcal{O}(\sqrt{T})$ w.r.t. $\mathcal{V}_0(\mathcal{E}_2)$ while the $\varepsilon$-greedy algorithm has linear regret lower bound $\Omega(T)$ w.r.t. $\mathcal{V}_0(\mathcal{E}_2)$.
\end{itemize}
We present the mean cumulative reward results in figure \ref{fig:mean_cum_rew}. We demonstrate the behavior of the algorithm on a single run in figure \ref{fig:master-1_run} in appendix \ref{section:experimental_details}. We provide additional details about the experimental setting in appendix \ref{section:experimental_details}.

\begin{figure}
\centering
\begin{subfigure}{.5\textwidth}
  \centering
  \includegraphics[width=1\linewidth]{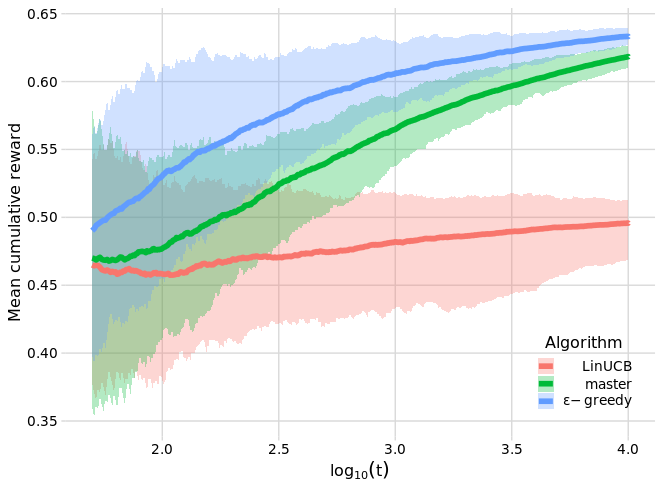}
\caption{Environment 1}
  \label{subfig:mean_cum_rew-env1}
\end{subfigure}%
\begin{subfigure}{.5\textwidth}
  \centering
  \includegraphics[width=1\linewidth]{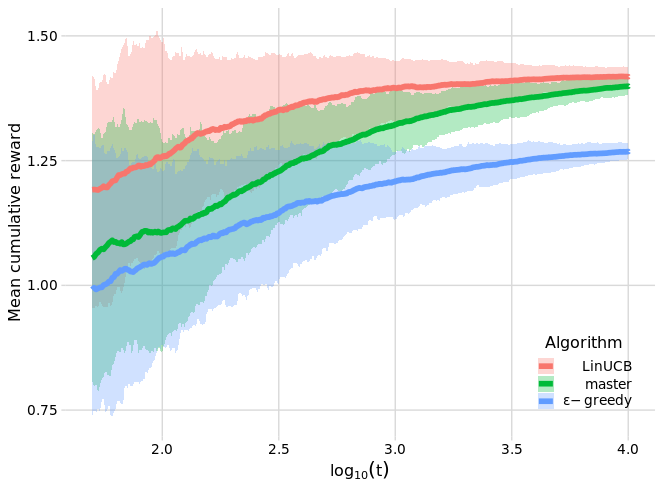}
  \caption{Environment 2}
  \label{subfig:mean_cum_rew-env2}
\end{subfigure}
\caption{Mean cumulative reward of the master and base algorithms over 100 runs, with (10\%,90\%) quantile bands}
\label{fig:mean_cum_rew}
\end{figure}

\section{Discussion}

We provided and analyzed a meta-learning algorithm that is the first proven rate-adaptive model selection algorithm for a general collection of contextual bandit algorithms. The general idea can be expressed in extremely simple terms: compare the performance of base learners at the same internal sample size, and explore uniformly at random with a well chosen decaying rate. Simulations confirm the validity of the procedure.

We commented on the nature of the guarantees of our algorithms and of previous approaches, and argued that they are close to (or coincide, under certain conditions, with) minimax adaptivity guarantees. We believe that further efforts should aim to bring the guarantees of model selection procedures under bandit feedback on par with the guarantees of cross-validation in the full-information setting. This would entail proving asymptotic equivalence with an oracle, which is an instance-dependent form of optimality, as opposed to minimax adaptivity.

\section*{Broader Impact}

Our work concerns the design of model selection / ensemble learning methods for contextual bandits. As it has the potential to improve the learning performance of any system relying on contextual bandits, it can impact essentially any setting where contextual bandits are used.

Contextual bandits are used or envisioned in settings as diverse as clinical trials, personalized medicine, ads placement and recommender systems. We therefore believe the broader impact of our work is positive inasmuch as these applications benefit to society.

\bibliography{biblio}

\begin{thebibliography}{17}
\providecommand{\natexlab}[1]{#1}
\providecommand{\url}[1]{\texttt{#1}}
\expandafter\ifx\csname urlstyle\endcsname\relax
  \providecommand{\doi}[1]{doi: #1}\else
  \providecommand{\doi}{doi: \begingroup \urlstyle{rm}\Url}\fi

\bibitem[Abbasi-Yadkori et~al.(2011)Abbasi-Yadkori, P\'{a}l, and
  Szepesv\'{a}ri]{abbasi-yadkori2011}
Y.~Abbasi-Yadkori, D.~P\'{a}l, and C.~Szepesv\'{a}ri.
\newblock Improved algorithms for linear stochastic bandits.
\newblock pages 2312--2320, 2011.

\bibitem[Agarwal et~al.(2014)Agarwal, Hsu, Kale, Langford, Li, and
  Schapire]{agarwalb14}
A.~Agarwal, D.~Hsu, S.~Kale, J.~Langford, L.~Li, and R.~E. Schapire.
\newblock Taming the monster: A fast and simple algorithm for contextual
  bandits.
\newblock In E.~P. Xing and T.~Jebara, editors, \emph{Proceedings of the 31st
  International Conference on Machine Learning}, volume~32 of \emph{Proceedings
  of Machine Learning Research}, pages 1638--1646, Beijing, China, 2014. PMLR.

\bibitem[Agarwal et~al.(2017)Agarwal, H., Neyshabur, and Schapire]{agarwal17}
A.~Agarwal, Luo H., B~Neyshabur, and R.~E. Schapire.
\newblock Corralling a band of bandit algorithms.
\newblock In Satyen Kale and Ohad Shamir, editors, \emph{Proceedings of the
  2017 Conference on Learning Theory}, volume~65 of \emph{Proceedings of
  Machine Learning Research}, pages 12--38, Amsterdam, Netherlands, 07--10 Jul
  2017. PMLR.

\bibitem[Benkeser et~al.(2018)Benkeser, Ju, Lendle, and van~der
  Laan]{benkeser2018}
D.~Benkeser, C.~Ju, S.~Lendle, and M.~J. van~der Laan.
\newblock Online cross-validation-based ensemble learning.
\newblock \emph{Statistics in Medicine}, 37\penalty0 (2):\penalty0 249--260,
  2018.

\bibitem[Bibaut et~al.(2020)Bibaut, Chambaz, and van~der Laan]{bibaut2020}
A.~F. Bibaut, A.~Chambaz, and M.~J. van~der Laan.
\newblock Generalized policy elimination: an efficient algorithm for
  nonparametric contextual bandits, 2020.

\bibitem[Chatterji et~al.(2019{\natexlab{a}})Chatterji, Pacchiano, and
  Bartlett]{chatterji19a}
N.~Chatterji, A.~Pacchiano, and P.~Bartlett.
\newblock Online learning with kernel losses.
\newblock In K.~Chaudhuri and R.~Salakhutdinov, editors, \emph{Proceedings of
  the 36th International Conference on Machine Learning}, volume~97 of
  \emph{Proceedings of Machine Learning Research}, pages 971--980, Long Beach,
  California, USA, 09--15 Jun 2019{\natexlab{a}}. PMLR.

\bibitem[Chatterji et~al.(2019{\natexlab{b}})Chatterji, Muthukumar, and
  Bartlett]{chatterji2019osom}
N.~S. Chatterji, V.~Muthukumar, and P.~L. Bartlett.
\newblock Osom: A simultaneously optimal algorithm for multi-armed and linear
  contextual bandits, 2019{\natexlab{b}}.

\bibitem[Devroye and Lugosi(2001)]{devroye-lugosi2001}
L.~Devroye and G.~Lugosi.
\newblock \emph{Combinatorial Methods in Density Estimation}.
\newblock Springer-Verlag, New York, 2001.

\bibitem[Dudoit and van~der Laan(2005)]{dudoit_vdL2005}
S.~Dudoit and M.~J. van~der Laan.
\newblock Asymptotics of cross-validated risk estimation in estimator selection
  and performance assessment.
\newblock \emph{Statistical Methodology}, 2:\penalty0 131--154, July 2005.

\bibitem[Foster et~al.(2019)Foster, Krishnamurthy, and Luo]{foster2019}
D.~J. Foster, A.~Krishnamurthy, and H.~Luo.
\newblock Model selection for contextual bandits.
\newblock In \emph{Advances in Neural Information Processing Systems 32}, pages
  14741--14752. Curran Associates, Inc., 2019.

\bibitem[Giné and Nickl(2015)]{gine_nickl_2015}
E.~Giné and R.~Nickl.
\newblock \emph{Mathematical Foundations of Infinite-Dimensional Statistical
  Models}.
\newblock Cambridge Series in Statistical and Probabilistic Mathematics.
  Cambridge University Press, 2015.

\bibitem[Gyorfi et~al.(2002)Gyorfi, Kohler, A., and H.]{gyorfi2002}
L.~Gyorfi, M.~Kohler, Krzyżak A., and Walk H.
\newblock \emph{A Distribution-free Theory of Nonparametric Regression}.
\newblock Springer-Verlag, New York, 2002.

\bibitem[Lepski(1990)]{Lepski90}
O.~V. Lepski.
\newblock A problem of adaptive estimation in gaussian white noise.
\newblock \emph{Theory of Probability and its Applications}, 35:\penalty0
  454--470, 1990.

\bibitem[Lepski(1991)]{Lepski91}
O.V. Lepski.
\newblock Asymptotically minimax adaptive estimation i: Upper bounds. optimally
  adaptive estimates.
\newblock \emph{Theory of Probability and its Applications}, 36:\penalty0
  682--697, 1991.

\bibitem[Massart(2007)]{massart2007}
P.~Massart.
\newblock \emph{Concentration inequalities and model selection}, volume 1896 of
  \emph{Lecture Notes in Mathematics}.
\newblock Springer, Berlin, 2007.

\bibitem[Pacchiano et~al.(2020)Pacchiano, M., Abbasi-Yadkori, A., Zimmert,
  Lattimore, and Szepesvari]{pacchiano2020model}
A.~Pacchiano, Phan M., Y.~Abbasi-Yadkori, Rao A., J.~Zimmert, T.~Lattimore, and
  C.~Szepesvari.
\newblock Model selection in contextual stochastic bandit problems, 2020.

\bibitem[Stone(1974)]{Stone74}
M.~Stone.
\newblock Cross-validatory choice and assessment of statistical predictions.
\newblock \emph{Journal of the Royal Statistical Society. Series B
  (Methodological)}, 36\penalty0 (2):\penalty0 111--147, 1974.

\end{thebibliography}

\appendix

\section{Proof of theorem \ref{thm:regret_master}}

We can without loss of generality assume that the tuning parameters $c_1$ and $c_2$ are set to 1.
The proof of theorem \ref{thm:regret_master} relies on the following lemmas.

\begin{lemma}\label{lemma:indep}
For any $j \in [J]$, and $n, t \geq $, $\underline{n}^{\xplr}(t)$ and $\widetilde{O}(j,n)$ are independent.
\end{lemma}

The following lemma tells us that the probability of selecting $j$ outside of the set  $\mathcal{J}(1)$ of optimal candidates decrease exponentially with the common internal time of candidates.

\begin{lemma}[Probability of selecting a suboptimal candidate]\label{lemma:proba_suboptimal_selector}
For all $n \geq 1$ and all $j \in [J] \backslash \mathcal{J}(1)$,
\begin{align}
P \left[ \widehat{j}(n) = j \right] \leq C_{3,j} \exp\left(-C_{4,j} n^{\kappa_j}\right),
\end{align}
with $C_{3,j},C_{4,j} > 0$ depending only on the constants of the problem, and $\kappa_j \in [0,1]$.
\end{lemma}

\begin{proof}
Suppose that $\widehat{j}(n) = j \in [J] \backslash \mathcal{J}(1)$. Then
\begin{align}
\widehat{R}(j^*,n) + n^{-\beta(1)} \geq \widehat{R}(j,n) + n^{-\beta_j},
\end{align}
which we can rewrite as
\begin{align}
&\left(\widebar{R}(j^*,n) - R^* \right) + \left(\widehat{R}(j^*,n) - \widebar{R}(j^*,n) \right) + \left( \widebar{R}(j,n) - \widehat{R}(j,n) \right) \\
\geq & \left(\widebar{R}(j,n) - R^*_j \right) + \left(R^*_j - R^*\right) + n^{-\beta_j} - n^{-\beta(1)}.
\end{align}
Using that $\widebar{R}(j,n) - R^*_j  \geq 0$, we must then have
\begin{align}
&\left(\widebar{R}(j^*,n) - R^* \right) + \left(\widehat{R}(j^*,n) - \widebar{R}(j^*,n) \right) + \left( \widebar{R}(j,n) - \widehat{R}(j,n) \right) \\
\geq & \left(R^*_j - R^*\right) + n^{-\beta_j} - n^{-\beta(1)}. \label{eq:basic_ineq_selector}
\end{align}
We distinguish two cases. 
\paragraph{Case 1: $j \not\in \mathcal{J}$.} Then, $R^*_j  -R^* \geq \Delta := \min_{j \not\in \mathcal{J}} R^*_j - R^*$, which is strictly positive by definition of $\mathcal{J}$. Denote $\gamma(1) := \max \{\gamma_j : j \in \mathcal{J}(1) \}$ Therefore, for $n \geq n_0$ for some $n_0$ depending only of $\Delta$ and $n^{-\beta(1)}$, we can lower bound the right-hand side of \eqref{eq:basic_ineq_selector} by $\Delta/2$, and we then have that for $n \geq n_0$,
\begin{align}
P \left[ \widehat{j}(n) = j \right] \leq & P \left[ \left(\widebar{R}(j^*,n) - R^* \right) + \left(\widehat{R}(j^*,n) - \widebar{R}(j^*,n) \right) + \left( \widebar{R}(j,n) - \widehat{R}(j,n) \right) \geq \frac{\Delta}{2} \right] \\
\leq & P \left[ \widebar{R}(j^*,n) - R^* \geq \frac{\Delta}{6} \right] + P\left[\widehat{R}(j^*,n) - \widebar{R}(j^*,n) \geq \frac{\Delta}{6} \right]+ P\left[ \widebar{R}(j,n) - \widehat{R}(j,n)  \geq \frac{\Delta}{6} \right].
\end{align}
From assumption \ref{assumption:concentration}, the first term can be bounded as follows:
\begin{align}
P \left[ \widebar{R}(j^*,n) - R^* \geq \frac{\Delta}{6} \right] \leq & C_1 \exp\left(-C_2 \left(n \left(\frac{\Delta}{6} - C_0 n^{-\beta(1)} (\log n)^{\gamma(1)} \right)^{1/\beta(1)}\right)^{\nu_{j^*}} \right) \\
\leq & \widetilde{C}_{3,j} \exp\left(-\widetilde{C}_{4,j} n^{\nu_{j^*}} \right),
\end{align}
for some $\widetilde{C}_{3,j} >0$ and $\widetilde{C}_{4,j} > 0$ that depend only on the constants of the problem.

The other two terms can be upper bounded using Azuma-Hoeffding: observing that for all $j'$, $(\widehat{R}(j',\tau) - \widebar{R}(j^*,\tau))_{\tau \geq 1}$ is a martingale difference sequence, and that from assumption \ref{assumption:conditional_sub_gaussian_rewards}, each of its term is 1-subGaussian conditionally on the past, we have that
\begin{align}
P\left[ \widehat{R}(j^*,n) - \widebar{R}(j^*,n) \geq \frac{\Delta}{6} \right] \leq \exp\left(-n \frac{\Delta^2}{36} \right) \qquad \text{and} \qquad P\left[ \widebar{R}(j,n) - \widehat{R}(j,n) \geq \frac{\Delta}{6} \right] \leq \exp\left(-n \frac{\Delta^2}{36} \right).
\end{align}
Therefore, $P[ \widehat{j}(n) = j] \leq C_{3,j} \exp(-C_{4,j} n^{\kappa_j})$, for some $C_{3,j}, C_{4,j} > 0$ that only depend on the constants of the problem, and $\kappa_j := \min(\nu_{j^*}, 1)$.

\paragraph{Case 2: $j \in \mathcal{J} \backslash \mathcal{J}(1)$.} Then $R^*_j  -R^* = 0$. As $j \in \mathcal{J} \backslash \mathcal{J}(1)$, we have $\beta_j < \beta(1)$, and therefore, for $n \geq n_{1,j}$ that depends only on $\beta(1)$ and $\beta_j$, we have $n^{-\beta_j} - n^{-\beta(1)} \geq n^{-\beta_j}/2$. For any $n \geq n_{1,j}$, we can then lower bound the right-hand side of \eqref{eq:exp_deviation_bound} by $n^{-\beta_j}/2$, and therefore, reasoning as in the previous step, we have that
\begin{align}
&P \left[ \widehat{j}(n) = j \right] \\
\leq & P \left[ \widebar{R}(j^*,n) - R^* \geq \frac{n^{-\beta_j}}{6} \right] + P\left[\widehat{R}(j^*,n) - \widebar{R}(j^*,n) \geq \frac{n^{-\beta_j}}{6} \right]+ P\left[ \widebar{R}(j,n) - \widehat{R}(j,n)  \geq \frac{n^{-\beta_j}}{6} \right] 
\end{align}
From assumption \ref{assumption:concentration}, the first term can be bounded as follows:
\begin{align}
P \left[ \widebar{R}(j^*,n) - R^* \geq \frac{1}{6} n^{-\beta_j} \right] 
\leq C_1 \exp\left(-C_2 \left( n \left(\frac{1}{6} n^{-\beta_j} - C_0 n^{-\beta(1)} (\log n)^{\gamma(1)} \right)^{1/\beta(1)} \right)^{\nu_j} \right).
\end{align}
For $n$ large enough, $n^{-\beta_j} / 6 - C_0 n^{-\beta(1)} (\log n)^{\gamma(1)} \geq n^{-\beta_j} / 12$, and therefore, there exists $\widetilde{C}_{3,j}$ and $\widetilde{C}_{4,j} > 0$ that depends only on the constants of the problem such that 
\begin{align}
P \left[ \widebar{R}(j^*,n) - R^* \geq \frac{1}{6} n^{-\beta_j} \right] \leq \widetilde{C}_{3,j} \exp \left(- \widetilde{C}_{4,j} n^{(1 - \beta_j / \beta(1))\nu_j} \right).
\end{align}
Using Azuma-Hoeffding as in case 1 yields that 
\begin{align}
&P\left[ \widehat{R}(j^*,n) - \widebar{R}(j^*,n) \geq \frac{n^{-\beta_j}}{6} \right]  \leq \exp\left(-\frac{n^{1-2\beta_j}}{36}\right),\\
 \text{and } &P\left[ \widebar{R}(j,n) - \widehat{R}(j,n) \geq \frac{n^{-\beta_j}}{6} \right] \leq  \exp\left(-\frac{n^{1-2\beta_j}}{36}\right).
\end{align}
Therefore, $P[\widehat{j}(n) = j] \leq C_{3,j} \exp(-C_{4,j} n^{\kappa_j})$, with $\widetilde{C}_{3,j}, \widetilde{C}_{4,j} > 0$ depending only on the constants of the problem, and $\kappa_j := \min(1 - 2 \beta_j, (1 - \beta_j / \beta(1)) \nu_j )$. Observe that $\kappa_j > 0$ as $\beta_j < \beta(1) \leq 1/2$ and $\nu_j > 0$.
\end{proof}

We can now prove theorem \ref{thm:regret_master}.

\begin{proof}[Proof of theorem \ref{thm:regret_master}]
Observe that the regret at time of the master w.r.t. $R^*$   can be decomposed as
\begin{align}
\mathrm{Reg}(t) :=& E \left[\frac{1}{t} \sum_{\tau=1}^t Y(\tau) \right] - R^* \\
=& E\left[ \sum_{j \in \mathcal{J}(1)} \frac{n(j,t)}{t} \left(\widebar{R}(j,n(j,t)) - R^* \right) \right] + E \left[ \sum_{j \not\in \mathcal{J}(1)} \frac{n(j,t)}{t} \left(\widebar{R}(j, n(j,t)) - R^*\right) \right].
\end{align}
Observe that for all $1 \leq n \leq t$, $n \widebar{R}(j,n) -R^*= \sum_{\tau=1}^n R(\pi(j,\tau) -R^* \leq \sum_{\tau=1}^t R(\pi(j,\tau)) - R^* = t \widebar{R}(j,t) -R^*$, since the terms in the sums are non-negative. Also, note that $\widebar{R}(j,n(j,t)) - R^* \leq 1$ for all $t$ and $j$. Therefore,
\begin{align}
\mathrm{Reg}(t) \leq \sum_{j \in \mathcal{J}(1)} E \left[\widebar{R}(j,t) - R^* \right] + \sum_{j \not\in \mathcal{J}(1)} \frac{E[n(j,t)]}{t}.
\end{align}
Recall that $n(j,t) = n^{\xplr}(j,t) + n^{\xplt}(j,t)$. It is straightforward to check that $E[n(j,t)^{\xplr}(t)] \leq t^{-\overline{\beta}}/(1-\overline{\beta}) $. We now turn to $E[n(j,t)^{\xplt}(t)]$. We have that
\begin{align}
E[n(j,t)^{\xplt}(t)] =& E\left[\sum_{\tau=1}^t \Ind(\widehat{J}(t) = j, D(\tau) =1) \right] \\
=& E \left[ \sum_{\tau=1}^t \Ind(\widehat{j}(\underline{n}^{\xplr}(\tau) =j, D(\tau) = 1) \right]\\
\leq & E \left[ \sum_{\tau=1}^t \Ind(\widehat{j}(\underline{n}^{\xplr}(\tau) =j) \right] \\
=& \sum_{\tau=1}^t E \left[ P \left[ \widehat{j}(\underline{n}^{\xplr}(\tau) =j \middle| \underline{n}^{\xplr}(\tau) \right]\right] \\
=& \sum_{\tau=1}^t E \left[ \sum_{n=1}^\tau P \left[ \widehat{j}(n) = j \middle| \underline{n}^{\xplr}(\tau) = n \right] \Ind(\underline{n}^{\xplr}(\tau) = n) \right]
\end{align}
From lemma \ref{lemma:indep}, $\widehat{j}(n)$ is independent of $\underline{n}^{\xplr}(\tau)$, and therefore, $P [ \widehat{j}(n) = j | \underline{n}^{\xplr}(\tau) = n ] = P [ \widehat{j}(n) = j ]$. Therefore, using this fact and the bound on $P [ \widehat{j}(n) = j ]$ from lemma 
\begin{align}
E[n(j,t)^{\xplt}(t)] \leq & \sum_{\tau=1}^t \sum_{n=1}^\tau P\left[\widehat{j}(n) = j \right] P \left[\underline{n}^{\xplr}(\tau) = n\right] \\
\leq & \sum_{\tau=1}^t \sum_{n=1}^\tau C_{3,j} \exp\left(-C_{4,j} n^{\kappa_j} \right) P \left[ \underline{n}^{\xplr(\tau)}=n\right] \\
=& \sum_{\tau=1}^t C_{3,j} E\left[ \exp\left(-C_{4,j} \underline{n}^{\xplr}(\tau) \right) \right].
\end{align}
It is straightforward to check that if $\kappa \in [0,1]$, $x \mapsto \exp(-x^\kappa)$ is convex. Therefore, from Jensen's inequality
\begin{align}
E[n(j,t)^{\xplt}(t)] \leq & C_{3,j} \sum_{\tau=1}^t \exp \left( -C_{4,j} E\left[\underline{n}^{\xplr}(\tau) \right] \right) \\
\leq & C_{3,j} \sum_{\tau=1}^t \exp\left(-\frac{C_{4,j}}{J(1 - \overline{\beta})} \tau^{-\overline{\beta}} \right) \\
\leq & C_{5,j},
\end{align}
with $C_{5,j} := C_{3,j} \int_0^\infty \exp\left(-\frac{C_{4,j}}{J(1 - \overline{\beta})} \tau^{-\overline{\beta}} \right) d \tau < \infty$.

Therefore, adding up the bounds on the expected number of exploration and exploitation rounds, we obtain
\begin{align}
E[n(j,t)] \leq \left(  C_{5,j} + \frac{1}{J(1-\overline{\beta})} t^{1-\overline{\beta}} \right) \leq C_{6,j} t^{1-\overline{\beta}},
\end{align}
for some $C_{6,j} > 0$ that depends only on the constants of the problem.

From assumption \ref{assumption:concentration}, for any $j \in \mathcal{J}(1)$, we have $E[\widebar{R}(j,t) - R^*] \leq C_7 t^{-\beta(1)}$ for some $C_6 > 0$. Therefore,
\begin{align}
\mathrm{Reg}(t) \leq & J C_6 t^{-\beta(1)} + \sum_{j=1}^J C_{7,j} t^{-\overline{\beta}} \\
\leq & C t^{-\beta(1)},
\end{align}
for some $C> 0$ that depends only on the constants of the problem.
\end{proof}

\section{Proof of the independence lemma}

We start by stating a more general result of which lemma \ref{lemma:indep} is a corollary.

\begin{lemma}\label{lemma:general_indep_result}
Consider some $j \in [J]$. 
Let, for all $t \geq 1$, $U(t) := \Ind(D(t)=1, J(t) = j)$. Then, for every $n,t \geq 1$, $U(t) {\perp\!\!\!\perp} \widetilde{\mathcal{F}}(j,n)$.
\end{lemma}

We can now prove lemma \ref{lemma:indep}. We relegate the proof of lemma \ref{lemma:general_indep_result} after the one of lemma \ref{lemma:indep}.

\begin{proof}[Proof of lemma \ref{lemma:indep}]
Observe that for every $j \in [J]$, $t \geq 1$, $n(j,t):= \sum_{\tau=1}^t \Ind(D(t)=1, J(t) = j)$. Lemma \ref{lemma:general_indep_result} then immediately gives the wished claim.
\end{proof}

\begin{proof}[Proof of lemma \ref{lemma:general_indep_result}]
Let for all $t\geq 1$,$\mathcal{F}^-(t):=\sigma(\mathcal{F}(t-1), D(t), J(t)).$
The hypothesis in the third bullet point can be rephrased as $Y(t)|\mathcal{F}^-(t) \stackrel{d}{=} \widetilde{Y}(j,n(j,t)) | \widetilde{\mathcal{F}}(j,n(j,t)-1)$.

Fix $j$ and $t$.  We denote $U(t):= \Ind(D(t) = 1, J(t)=j)$. Observe that $U(t)$ is $\mathcal{F}^-(t)$-measurable, and that from the first and second conditions, $U(t) {\perp\!\!\!\perp} \mathcal{F}(t-1)$.

 We show by induction that for all $n \geq 1$, $U(t) {\perp\!\!\!\perp} \widetilde{\mathcal{F}}(j,n)$. We treat the base case at the end of the proof. Suppose that for some $n\geq1$, $U(t) {\perp\!\!\!\perp} \widetilde{\mathcal{F}}(j,n)$. Let us show that  $U(t) {\perp\!\!\!\perp} \widetilde{\mathcal{F}}(j,n+1)$. It suffices to show that $U(t) {\perp\!\!\!\perp} \widetilde{Y}(j,n+1) | \widetilde{\mathcal{F}}(j,n)$. Observe that
\begin{align}
&P\left[\widetilde{Y}(j,n+1)=y,U(t)=u \middle| \widetilde{\mathcal{F}}(j,n)\right] \\
=& P\left[\widetilde{Y}(j,n+1)=y,U(t)=u, t(j,n+1)<t \middle| \widetilde{\mathcal{F}}(j,n)\right]  \\
&+ P\left[\widetilde{Y}(j,n+1)=y, U(t)=u, t(j,n+1) \geq t \middle| \widetilde{\mathcal{F}}(j,n)\right].
\end{align}
We start with the first term. We have that 
\begin{align}
&P\left[\widetilde{Y}(j,n+1)=y,U(t)=u, t(j,n+1)<t \middle| \widetilde{\mathcal{F}}(j,n)\right] \\
=& P \left[U(t)=u\middle| \widetilde{Y}(j,n+1)=y,t(j,n+1)<t \middle| \widetilde{\mathcal{F}}(j,n)\right] P \left[ \widetilde{Y}(j,n+1)=y,t(j,n+1)<t , \widetilde{\mathcal{F}}(j,n)\right] \\
=& P \left[U(t)=u\right] P \left[ \widetilde{Y}(j,n+1)=y,t(j,n+1)<t \middle| \widetilde{\mathcal{F}}(j,n)\right]
\end{align}
since $\{\widetilde{Y}(j,n)=y, t(j,n+1) < t \} \cap \widetilde{\mathcal{F}}(j,n)$ is $\mathcal{F}(t-1)$ measurable and $U(t) {\perp\!\!\!\perp} \mathcal{F}(t-1)$. Moreover, observe that $\{t(j,n+1) < t\} \cap \widetilde{\mathcal{F}}(j,n)$ is $\mathcal{F}^-(t(j,n+1))$-measurable, and therefore,
\begin{align}
&P \left[ \widetilde{Y}(j,n+1)=y \middle| t(j,n+1)<t, \widetilde{\mathcal{F}}(j,n)\right] \\
=& E \left[ P \left[ Y(t(j,n+1))=y \middle| \mathcal{F}^-(t(j,n+1))\right] \middle| t(j,n+1)<t, \widetilde{\mathcal{F}}(j,n)\right] \\
=& E \left[ P \left[\widetilde{Y}(j,n+1) \middle| \widetilde{\mathcal{F}}(j,n) \right] \middle| t(j,n+1) < t, \widetilde{\mathcal{F}}(j,n) \right] \\
=& P \left[\widetilde{Y}(j,n+1) \middle| \widetilde{\mathcal{F}}(j,n)\right].
\end{align}
Therefore,
\begin{align}
&P\left[\widetilde{Y}(j,n+1)=y,U(t)=u, t(j,n+1)<t \middle| \widetilde{\mathcal{F}}(j,n)\right]  \\
=& P\left[\widetilde{Y}(j,n+1) \middle| \widetilde{\mathcal{F}}(j,n) \right] P \left[U(t)=u\right] P \left[ t(j,n+1) < t  \middle| \widetilde{\mathcal{F}}(j,n) \right]\\
=& P\left[\widetilde{Y}(j,n+1) \middle| \widetilde{\mathcal{F}}(j,n) \right] P \left[U(t)=u, t(j,n+1) < t  \middle| \widetilde{\mathcal{F}}(j,n) \right],
\end{align}
since $U(t) {\perp\!\!\!\perp} \{ t(j,n+1) < t\}$, as $\{ t(j,n+1) < t\}$ is $\mathcal{F}(t-1)$-measurable, and $U(t) {\perp\!\!\!\perp} \widetilde{\mathcal{F}}(j,n)$ by induction hypothesis, which imply that $U(t) {\perp\!\!\!\perp} \{ t(j,n+1)<t\} | \widetilde{\mathcal{F}}(j,n)$.

We now turn to the second term. Observe that $\{U(t)=u, t(j,n+1) \geq t \} \cap \mathcal{F}(j,n)$ is $\mathcal{F}^-(t(j,n+1))$-measurable. Therefore,
\begin{align}
&P \left[ \widetilde{Y}(j,n+1)=y \middle| U(t)=u, t(j,n+1) \geq t, \widetilde{\mathcal{F}}(j,n) \right] \\
=& E \left[ P \left[ Y(t(j,n+1))=y \middle| \mathcal{F}^-(t(j,n+1)) \right] \middle|  U(t)=u, t(j,n+1) \geq t, \widetilde{\mathcal{F}}(j,n) \right] \\
=& E \left[ P \left[ \widetilde{Y}(j,n+1)=y \middle| \widetilde{\mathcal{F}}(j,n) \right] \middle| U(t)=u, t(j,n+1) \geq t, \widetilde{\mathcal{F}}(j,n+1) \right] \\
=& P \left[ \widetilde{Y}(j,n+1)=y \middle| \widetilde{\mathcal{F}}(j,n) \right].
\end{align}
Therefore, 
\begin{align}
&P\left[\widetilde{Y}(j,n+1)=y, U(t)=u, t(j,n+1) \geq t \middle| \widetilde{\mathcal{F}}(j,n)\right] \\
=& P \left[\widetilde{Y}(j,n+1)=y \middle| \widetilde{\mathcal{F}}(j,n) \right] P\left[U(t)=u, t(j,n+1) \geq t \middle| \widetilde{\mathcal{F}}(j,n) \right].
\end{align}
Therefore, adding up the identities for the two terms,  we have 
\begin{align}
P\left[\widetilde{Y}(j,n+1)=y, U(t)=u \middle| \widetilde{\mathcal{F}}(j,n)\right] = P \left[\widetilde{Y}(j,n+1)=y \middle| \widetilde{\mathcal{F}}(j,n) \right] P\left[U(t)=u \middle| \widetilde{\mathcal{F}}(j,n) \right].
\end{align}
We have thus shown that $\widetilde{Y}(j,n+1) {\perp\!\!\!\perp} U(t) | \widetilde{\mathcal{F}}(j,n)$, which implies that $U(t) {\perp\!\!\!\perp} \widetilde{\mathcal{F}}(j,n+1)$.

The base case can be treated with the same arguments.
\end{proof}

\section{Proofs of reformulations of regret bounds for known base algorithms}

\begin{proof}[Proof of corollary \ref{corollary:exp_dev_UCB}]
As $\widetilde{Y}_{\tau} - E[\widetilde{Y}_\tau| \widetilde{F}_{\tau - 1} ]$ is conditionally $1$-sub-Gaussian with probability at least $1 - \delta / 2$,
\begin{align}
\mathrm{CondReg}(n) \leq \mathrm{Reg}(n) + \sqrt{n \log(2/\delta)},
\end{align}
and thus, using the high-probability regret bound from \cite[lemma 4.9 in]{pacchiano2020model}, there exists $C > 0$ such that, with probability at least $1 - \delta$, 
\begin{align}
\mathrm{CondReg}(n) \leq & C \sqrt{n \log(2 n / \delta)} + \sqrt{n \log(2 / \delta)} \\
\leq & C \sqrt{n \log (2n)}  +  (C +1)\sqrt{n \log (2 /\delta)} \\
\leq & C' \sqrt{n \log n} + C' \sqrt{n \log (1 / \delta)},
\end{align}
for some $C' > C + 1$. 
Let $x =  C \sqrt{n \log(1/ \delta)}$, that is $\delta =\exp(-(C')^{-2} n x^2)$. Recalling that $\widebar{R}(n) - R^* = \mathrm{CondReg}(n) / n$, we thus have that 
\begin{align}
P \left[\widebar{R}(n) -R^* \geq C' n^{-1/2} (\log n)^{1/2} +x \right] \leq \exp\left(-(C')^{-2} n x^2 \right).
\end{align}
\end{proof}

\begin{proof}[Proof of lemma \ref{lemma:exp_dev_eps_greedy}] It suffices to observe that 
\begin{enumerate}
\item The bracketing entropy in any $L_p$ norm is always dominated by the covering entropy in $\|\cdot\|_\infty$ norm.
\item The proof of \cite[theorem 2 in][]{bibaut2020} gives the desired bound on $\widebar{R}(n) - R^*$ as an intermediate result (right before relating it to the regret by using Azuma-Hoeffding).
\end{enumerate}
\end{proof}

\begin{proof}[Proof of corollary \ref{corollary:exp_dev_linucb}] \cite[Theorem 3 in][]{abbasi-yadkori2011} gives that there exists $C > 0$ such that $\widebar{R}(n) - R^* = \mathrm{CondReg}(n) / n \leq C n^{-1/2} \log(1 / \delta)$ with probability at least $1 - \delta$. 
Setting $x = C n^{-1/2} \log (1 / \delta)$, that is $\delta := \exp(-C^{-1} \sqrt{n} x)$, we have that
\begin{align}
P \left[ \widebar{R}(n) - R^* \geq x \right] \leq \exp\left(- C^{-1} \sqrt{n} x \right),
\end{align}
which is the wished claim.
\end{proof}

\begin{proof}[Proof of corollary \ref{corollary:exp_dev_agarwal}]
As $\widehat{Y}_\tau - E[\widehat{Y}_\tau | \widehat{\mathcal{F}}_{\tau-1}]$ is conditionally $1$-sub-Gaussian, Azuma-Hoeffding gives us that, with probability at least $1- \delta/2$,
\begin{align}
\mathrm{CondReg}(n) \leq \mathrm{Reg}(n) + \sqrt{n \log(2/\delta)}.
\end{align}
Therefore, combining this with the claim of \cite[theorem 2 in][]{agarwalb14}, there exists $C > 0$, such that, with probability $1 - \delta$
\begin{align}
\widebar{R}(n) -R^* \leq & C n^{-1/2} \sqrt{\log (2n / \delta)} + Cn^{-1} \log(2n / \delta) + n^{-1/2} \sqrt{\log(2 / \delta)} \\ 
\leq & C \left(n^{-1/2} \sqrt{\log n} + n^{-1} \log n \right) + C n^{-1} \log (2/\delta) + (C+1) n^{-1/2} \sqrt{\log(2/\delta)} \\
\leq & C' \left(n^{-1/2} \sqrt{\log n} + n^{-1/2} \log(1 / \delta) \right),
\end{align}
for some $C' > C$.
Letting $x=C' n^{-1/2} \log (1 / \delta)$, this is equivalent with
\begin{align}
P \left[\widebar{R}(n)- R^* \geq C' n^{-1/2} (\log n)^{1/2} + x \right] \leq \exp\left(-C' \sqrt{n} x \right),
\end{align}
which is the wished claim.
\end{proof}

%

\section{Comment on the need to enforce a lower bound on the estimated risk}\label{section:lower_bound}

Unlike model selection methods such as Lepski's method and cross-validation, our method relies on explicit identification of the index of the best model. It is our understanding that such index identification tasks usually require the existence of a lower bound on the risk of each alternative, so as to ensure a gap in performance between the best and second best learner. Consider for instance the situation where one wants to adaptively estimate in $L_\infty$ norm a density belonging to the union of a collection of Holder balls: $\mathcal{M}_s = H(s, B)$, where $H(s, M):= \{f:\mathbb{R}^d \to \mathbb{R}: |f(x) - f(y)| \leq M |x - y|^{s-\lfloor s \rfloor } \}$. It is well known, that while Lepski	's method is a minimax adaptive procedure with respect to $\{ \mathcal{M}_s : s \in \mathcal{S} \}$, identification of the index $s$ of the smallest Holder class that contains the truth is impossible without additional assumptions that enforce risk lower bounds \citep{gine_nickl_2015}.

A parallel can perhaps be drawn with the best arm identification problem in multi-armed bandit settings: the analysis relies on the  gap in mean reward between the best and second best arm.

Lower bounds of the sort we enforce are intrinsically tied to the minimax framework: they require the knowledge of a rate associated to the model class. Moving beyond the minimax framework to design a meta-learner that performs as well as the best instance-dependent base learner therefore seems to imply that such a procedure must not rely on identifying the index of the best model.

\section{Experimental details}\label{section:experimental_details}

In both environment, contexts are i.i.d. draws from $\mathcal{N}(\bm{0}, \textbf{I}_4)$. 

In environment $1$, the rewards Bernoulli conditional on $A$ and $X$, with conditional means specified as follows: for all $x=(x_1,x_2,x_3,x_4) \in \mathbb{R}^4$,
\begin{align}
E[Y|A=1, X=x] = \begin{cases}
0.1 \text{ if } x_1 < 0 \text{ and  } x_2 <0,\\
0.5 \text{ if } x_1 < 0 \text{ and } x_1 \geq 0,\\
0.7 \text{ if } x_1 \geq 0 \text{ and } x_2 < 0,\\
0.45 \text{ otherwise} 
\end{cases},
\end{align}
and
\begin{align}
E[Y|A=1, X=x] = \begin{cases}
0.8 \text{ if } x_1 < 0 \text{ and  } x_2 <0,\\
0.1 \text{ if } x_1 < 0 \text{ and } x_2 \geq 0,\\
0.3 \text{ if } x_1 \geq 0 \text{ and } x_2 < 0,\\
0.6 \text{ otherwise} 
\end{cases}.
\end{align}

In environment 2, for each $a \in \{1,2\}$, rewards are normally distributed conditional on $X$: $E[Y|A=a, X=x] = \mu_a(x) \rangle + \eta$, with $\eta \sim \mathcal{N}(0,1)$, with
\begin{align}
\mu_1(x) = & 0.9 + 0.5 x_1 + 0.3 x_2  -0.9 x_3  -0.2 x_4 \\
\mu_2(x) = & 0.9  -0.5 x_1 +  0.1 x_2 -0.7 x_3  + 0.6 x_4.
\end{align}

The $\varepsilon$-greedy learner uses as expected reward model the set of functions of the form $(a, x) \mapsto \beta_{a,0} + \beta_{a,1} \Ind(x_1 <0, x_2 < 0) + \beta_{a,2} \Ind(x_1 <0, x_2 \geq 0) + \beta_{a,3} \Ind(x_1 \geq 0, x_2 \geq 0)$. The reward learner therefore converges at a parametric rate to the truth under environment 1. Therefore, by setting the exploration rate to $t^{-1/3}$ at each $t$, regret under environment 1 is $\mathcal{O}(T^{2/3})$ over $T$ rounds. However, this reward model does not contain the truth under environment 2,which implies that the $\varepsilon$-greedy algorithm incurs linear regret w.r.t. $\mathcal{V}_0(\mathcal{E}_1)$.

We use LinUCB with a linear model including all four components of $x$ and an intercept, which implies that the realizability assumption is satisfied under environment 2, and therefore that the regret of LinUCB w.r.t. $\mathcal{V}_0(\mathcal{E}_2)$ is $\mathcal{O}(\sqrt{T})$ over $T$ rounds.

Figure \ref{fig:master-1_run} demonstrates the master algorithm and its two base learners on a single run.

\begin{figure}[H]
\centering
\begin{subfigure}{.5\textwidth}
  \centering
  \includegraphics[width=1\linewidth]{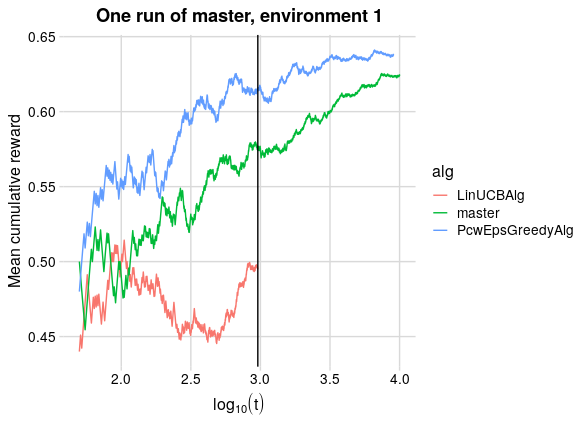}
  \label{fig:sub1}
\end{subfigure}%
\begin{subfigure}{.5\textwidth}
  \centering
  \includegraphics[width=1\linewidth]{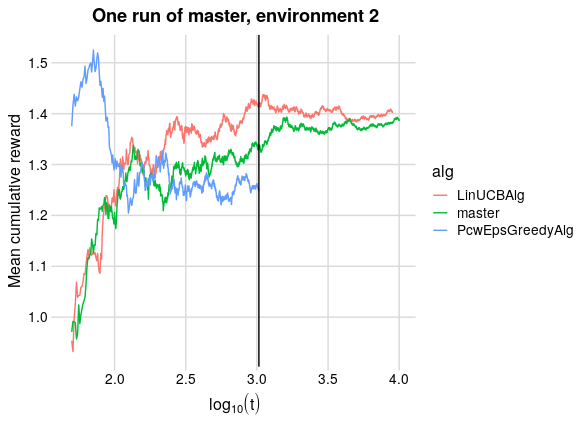}
  \label{fig:sub2}
\end{subfigure}
\caption{Mean cumulative reward of the master and its two base algorithms over 1 run. The vertical black line indicates $\underline{n}^{\xplr}(T)$, with $T$ the final global time in the simulation.}\label{fig:master-1_run}
\end{figure}

We tried the following values for the hyperparameters: $(c_1,c_2) \in \{0.1, 0.5, 1\} \times \{1, 10\}$. All specifications lead to the master appearing to converge to the performance of the optimal algorithm, but some values degrade a the performance in earlier rounds. As pointed out earlier, the specific constant values of $c_1$ and $c_2$ have no impact on the asymptotics.

The results of figure \ref{fig:mean_cum_rew} were generated using an AWS EC2 instance of type r4.8xlarge, with 32 cores and 244 GiB of memory. Each plot takes about 30 minutes to compute. 

The results of figure \ref{fig:master-1_run} were generated on a personal laptop and take less than 5 minutes to compute.
\end{document}